%% file: example_paper.tex
\documentclass{article}

\usepackage{microtype}
\usepackage{graphicx}
\usepackage{booktabs} 
\usepackage{subcaption}
\usepackage{hyperref}
\usepackage{multirow}
\usepackage{bm}
\usepackage{bbding}
\usepackage{enumitem}

\usepackage[accepted]{icml2025}

\usepackage{amsmath}
\usepackage{amssymb}
\usepackage{mathtools}
\usepackage{amsthm}
\usepackage{bm}

\usepackage[capitalize,noabbrev]{cleveref}

\theoremstyle{plain}
\newtheorem{theorem}{Theorem}[section]

\newtheorem{lemma}[theorem]{Lemma}

\theoremstyle{definition}

\theoremstyle{remark}

\newcommand{\ie}{\emph{i.e.,} }
\newcommand{\eg}{\emph{e.g.,} }
\newcommand{\paratitle}[1]{\vspace{1.5ex}\noindent\textbf{#1}}
\usepackage{lineno}
\newcommand{\tabincell}[2]{\begin{tabular}{@{}#1@{}}#2\end{tabular}}
\newcommand{\revision}[1]{\textcolor{black}{#1}}
\usepackage[textsize=tiny]{todonotes}

\icmltitlerunning{Submission and Formatting Instructions for ICML 2025}

\begin{document}

\twocolumn[
\icmltitle{Irrational Complex Rotations Empower Low-bit Optimizers}




\icmlsetsymbol{equal}{*}

\begin{icmlauthorlist}
\icmlauthor{Zhen Tian}{yyy}
\icmlauthor{Wayne Xin Zhao}{yyy}
\icmlauthor{Ji-Rong Wen}{yyy}
\end{icmlauthorlist}

\icmlaffiliation{yyy}{Gaoling School of Artificial Intelligence, Remin University of China}

\icmlcorrespondingauthor{Wayne Xin Zhao}{batmanfly@gmail.com}

\icmlkeywords{Machine Learning, ICML}

\vskip 0.3in
]



\printAffiliationsAndNotice{} 

\begin{abstract}

In this paper, we propose a novel optimizer state compression algorithm, namely \textbf{$\pi$-Quant}, which leverages the properties of irrational numbers (\eg $\pi$) for memory-efficient training. 
The core idea is based on our mathematical findings, which show that a pair of parameters can be represented by a single rotation angle using the complex rotation scheme.
Building on this insight, we map the parameters into a complex space and perform quantization using the corresponding rotation angles. 
To efficiently integrate it into optimization process, we develop an efficient system of geometric equations that computes the precise rotation angles with linear complexity.
\revision{We evaluate $\pi$-Quant on  a wide range of tasks. Our experiments show that it can reduce the bit-width of parameters to 3.32-bit, achieving a 75\% reduction in parameter scale and a 40\% decrease in GPU memory usage, all while maintaining full accuracy.} 


\end{abstract}

\input{sec-intro}
\input{sec-pre}
\input{sec-method}
\input{sec-exp}
\input{sec-related}
\input{sec-conclusion}

\bibliography{example_paper}
\bibliographystyle{icml2025}

\newpage
\appendix
\input{sec-appendix}


\end{document}

%% file: sec-intro.tex
\section{Introduction}
\label{sec:intro}
The rapid growth of large-scale AI models (\eg large language models~\cite{zhao2023survey}) has led to significant improvements in various tasks. 
\revision{
However, larger models often require substantially more computational power and memory resources to train.
Typically, most models are trained with momentum-based optimizers such as Adam~\cite{Adam2015}, which require the storage of one or two optimizer momentum for each individual parameter. 
As the parameter scale increases, the optimizer states consume the majority of memory, with first-order and second-order momentum tensors in the Adam accounting for 66\% of parameter usage~\cite{li2023memory}.}  Therefore, finding a solution to develop a memory-efficient optimizer is a critical challenge.



Considering the above issues, a number of studies~\cite{dettmers20158, li2023memory} propose using lower-precision representations for \revision{the optimizer states}, as shown in Table~\ref{tab:introcom}.
They primarily analyze the frequently occurring values of the optimizer states, and then use binary search to quantize all state elements to these values.
While previous attempts have demonstrated the promise of this approach, two major challenges remain.
First, 
most deep learning frameworks (\eg Pytorch~\cite{paszke2019pytorch}, TensorFlow~\cite{abadi2016tensorflow}) do not support such search quantization operations, these methods require additional compilation of GPU-supported operators.
As a result, these methods are not directly compatible with devices other than GPUs  (\eg CPU or TPU). 
Second, these methods are specifically designed for certain bit-width quantization, making it difficult to adapt them to other bit-widths. Since these predefined precisions may not be well-suited to all tasks, they can lead to performance degradation in certain cases. 

\begin{table}[!t]
  \centering
  \small
  \captionsetup{font={small}}
  \caption{Comparison of different quantization methods. Here, $n$ denotes the parameter size, and $m$ denotes the precision size.}  
  \label{tab:introcom}
  \resizebox{1.\columnwidth}{!}{
  \begin{tabular}{c|ccc}
    \toprule
    \textbf{Metric} & \textbf{Bnb (2021)} & \textbf{Lpmm (2023)} & \textbf{$\pi$-Quant (ours)} \\
    \hline \hline
    Bit-width & 8 & 4 & 3.32\\
    Complexity & $O(n\log m)$ & $O(n\log m)$ & $O(n)$\\
    Full Acc.  & \CheckmarkBold & \XSolidBrush & \CheckmarkBold\\
    \bottomrule
\end{tabular}}
\vspace{-2em}
\end{table}

To address these issues, we aim to develop a \emph{precise} and \revision{\emph{theoretically-guaranteed}} compression algorithm that can flexibly reduce \revision{the memory usage of optimizer states} within a typical model optimization framework.
Our key idea is grounded in an important mathematical property: $\forall \{(x, y) | x^2 + y^2 <= 4\}, \exists \theta \in \mathbb{R}, x + iy = e^{i\theta} + e^{i\pi \theta}$.
This enables the \emph{precise} representation of a parameter pair $(x, y)$ with a single complex rotation angle (\ie $\theta$), \revision{by harnessing the properties of irrational numbers (\ie $\pi$).} 
As a result, the parameter scale is halved with no loss of precision. 
Further, we can quantize the rotation angles to achieve even greater memory reduction. 
To implement our idea, a fundamental challenge is integrating the proposed complex rotation scheme into the optimization process. 
Since the optimizers are often designed for real-valued parameters, they are not directly compatible with our complex transformation.
Further, it is also difficult to find an efficient solution for accurately computing the rotation angles of the corresponding parameters during optimization.


Motivated by our mathematical findings, in this paper, we propose a novel optimizer state compression approach with irrational complex rotation, named \textbf{$\pi$-Quant}.
The major contribution of $\pi$-Quant is  the introduction of a precise theoretical framework for parameter compression, resulting in lower memory costs during model training.
Specifically, $\pi$-Quant involves two key techniques. 
First, it develops a new representation mechanism that maps the parameters in a complex space.
On this basis, the complex-valued parameters are transformed into single rotation angles according to our proposed mathematical formula, which precisely halves the parameter scale. 
Notably, it employs an efficient system of geometric equations that can compute the precise rotation angles with linear complexity.
Second, it features \revision{an effective quantization algorithm that reduces the precision of the rotation angles during training}, all while preserving full accuracy. 
Our theoretical framework is particularly robust for numerically sensitive parameters, offering lower quantization errors compared to prior methods. Generally, $\pi$-Quant can be seamlessly integrated into existing optimization pipeline, which largely reduces the training memory footprint, with minimal impact on the training speed.

The main contributions are summarized as follows:

$\bullet$ 
We propose an effective optimizer state compression approach, named $\pi$-Quant, by introducing a new mathematical formula. This method transforms a parameter pair into a univariate rotation equation with irrational coefficients, enabling precise compression by halving  the parameter scale. To the best of our knowledge, we are the first to leverage the properties of irrational numbers for model compression.


$\bullet$ We propose an efficient system of geometric equations for optimizing the compression process, which is capable of computing the precise rotation angle with linear complexity.
\revision{Additionally, we develop an effective quantization algorithm that reduces the representation precision of the rotation angles, while maintaining full accuracy.}

$\bullet$ 
Experimental results demonstrate that $\pi$-Quant is an effective quantizer, capable of reducing the bit-width of optimizer states to 3.32 bits. 
For example, it can reduce the training memory of TinyLlama from 19.47 G to 11.32 G, with comparable accuracy.
Besides, it consistently outperforms several state-of-the-art quantization methods on a wide range of tasks, highlighting the effectiveness of our approach.

%% file: sec-pre.tex
\section{Preliminaries}
In this section, we give a brief introduction to the quantization techniques, and then describe the complex rotation technique used in our approach.

\paratitle{Training-Oriented Quantization}. Quantization is a compression technique that maps high-precision values to a finite set of discrete values. 
The quantization methods can be generally divided into \emph{inference-oriented} and \emph{training-oriented} according to their application scenarios.
Typically, inference-oriented methods often conduct post-training quantization of the backbone network (\eg dense layers) for reducing inference costs.
For example, the INT8 method partitions the range of well-trained FP32 parameters into 256 equal length segments, and then map them into 8-bit integers based on their respective magnitudes.
Despite the progress, these methods are not practical to be applied into training pipeline.
Because these methods fail to effectively fit the non-uniform distributed parameters, they often suffer from significant accuracy loss during optimization.

Unlike inference-oriented quantization, training-oriented methods~\cite{dettmers20218, li2023memory} often focus on reducing the precision of optimizer states to achieve lower training memory. 
For better fitting the non-uniform parameter distributions, their basic idea is to sample high-frequency elements to construct a non-uniform quantization range $D$, \eg $D = \{0.12, 0.25, 0.5, 1.0\}$.
\revision{Due to such non-uniformity, these methods cannot directly compute the quantization indices and typically require a search-based method to perform the quantization:}  
\begin{align}\label{eq:pre}
\begin{aligned}
\textbf{X}^{INT}_i = \arg \min_{\theta \in D} ||\theta - \textbf{X}^{FP32}_i||.
\end{aligned}
\end{align}
In practice, these methods often require compiling additional GPU kernels to perform parallel searching.
Besides, the certain bit-width setting of $D$ cannot be extended to other bit-widths.
As a training-oriented method, our approach focus on reducing the memory of optimizers, which pushes the lowest achievable bit-width (4-bit) down to 3.32-bit, with higher accuracy and no need for a search process.

\paratitle{Complex Rotation}. In mathematics, complex rotation is a mathematical operation that applies a rotation in the complex plane. 
It is represented by multiplying a complex number of the form
\ie $\cos\theta + i \sin\theta$, where $i$ is the imaginary unit ($i^2 = -1$) and $\theta$ is the rotation angle.
According to Euler's formula, a complex number can be written as:
\begin{align}\label{eq:com}
\begin{aligned}
\cos\theta + i \sin\theta = e^{i\theta}.
\end{aligned}
\end{align}

This equation establishes a mapping from $\{(x, y) | x^2 + y^2 = 1\}$  to $\theta \in \mathbb{R}$. 
\revision{In our paper, we reveal a more profound theorem, which shows that arbitrary pair of $(x, y)$ can be represented by an angle $\theta$ in this complex rotation scheme.} 
As such, it provides a new solution to compress the parameters.

%% file: sec-method.tex
\section{Methodology}
\label{sec:method}
In this section, we introduce the proposed $\pi$-Quant method (illustrated in Figure~\ref{fig:model}) for achieving memory-efficient parameter optimization. Our method is grounded in a key mathematical property: each real vector can be split into two components, which can then be uniquely represented by a rotation angle. Building on this idea, we propose a geometric approach for performing this transformation, enabling the computation of the rotation angle with linear complexity. 
Moreover, we design a quantization pipeline to implement our method, which significantly reduces the parameter scale and can be seamlessly integrated into \revision{existing optimization algorithms}. We begin by introducing the theoretical framework of $\pi$-Quant, followed by a description of its practical application pipeline.



\begin{figure}[ht]
\vskip 0.2in
\begin{center}
\centerline{\includegraphics[width=\columnwidth]{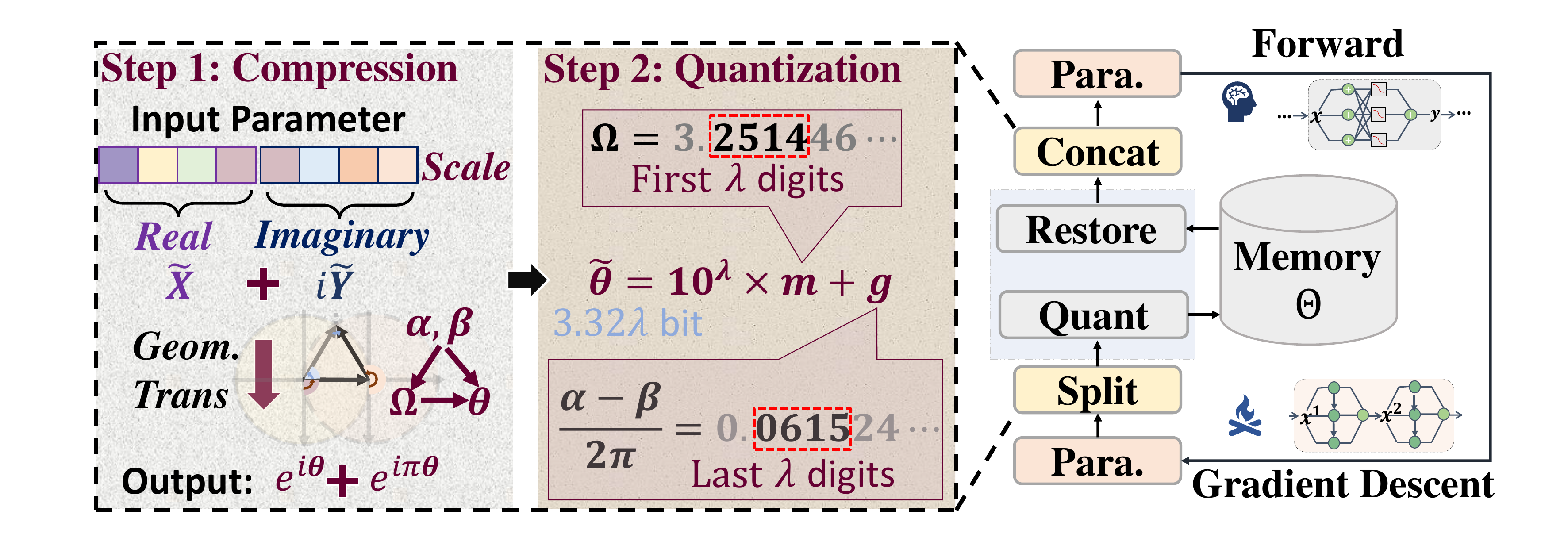}}
\caption{The overall framework of $\pi$-Quant.}
\label{fig:model}
\end{center}
\vskip -0.2in
\end{figure}

\subsection{Parameter Compression with Irrational Rotations}~\label{sec:irr}
\noindent{The key point of our approach is a novel data compression theorem, which enables the precise representation of two variables using a univariate rotation equation with an irrational number as the coefficient. We first present the mathematical foundation behind this theorem, followed by the instantiation of our idea.} 

\paratitle{Data Compression with Irrational Numbers.} Existing machine learning frameworks typically represent model parameters using floating-point numbers (\eg FP32) with limited precision. In this representation, previous work~\cite{frantar2022gptq} often employs truncation-like methods to uniformly quantize high-precision parameters to a lower precision. The underlying assumption is that the parameters follow a uniform distribution within the quantization range. However, due to the mismatch between this assumption and the actual distribution of the parameters, most quantization methods encounter significant errors, as parameters are more likely to follow a non-uniform distribution (\eg Gaussian distribution). 
In contrast to prior approaches, we introduce a novel data compression technique that compresses parameters from the perspective of irrational numbers. Formally, we present the following theorem:


\begin{theorem}
\label{thm:irr}
Given a complex number $z = x + iy \in \mathbb{C}$ that satisfies $||z|| \leq 2$, there exists a unique angle $\theta \in \mathbb{R}$ that rewrites $z$:

\begin{align}\label{eq:irr}
    z = x + iy = e^{i\theta} + e^{i\bar{\pi}\theta},
\end{align}
where $i$ is the imaginary unit, and $\bar{\pi}$ denotes any  
irrational number (e.g., the circle ratio $\pi$). 
\end{theorem}
Specifically, it shows that any vector of two real values $x$ and $y$ can be represented by a real-valued angle $\theta$ in this complex scheme.  This property extends to any parameter tensor $\mathbf{T}$, which can be split and linearized into two equal-sized matrices, \ie $(\bm{X}, \bm Y) = \text{Split}(\mathbf{T})$, corresponding to the real part and imaginary part, respectively.
In this manner, we can represent $\mathbf{T}$ by a specific angle vector $\bm{\Theta}$, using half the number of real values.
A proof of this theorem is provided in Appendix~\ref{sec:app1}.




\paratitle{Geometric Solutions.} 
Finding the exact solution to Eq.~\eqref{eq:irr} is the key challenge of our approach. 
Intuitively, one might consider using a search-based method (\eg KDTree), where  the $x$ and 
$y$ values corresponding to all $\theta$ angles within the desired precision are stored in a cached array, and the array is then searched for the coordinates closest to the target $(x, y)$.
However, this approach often incurs $O(n \log m)$ complexity and complexity and presents challenges in performing parallel searches on a GPU. 
As shown in Figure~\ref{fig:model}, to efficiently solve the Eq.~\eqref{eq:irr}, we propose a geometric solution that enables faster implementation: 
\begin{lemma}
\label{lem:geolemma}
Given the input complex number $x + iy$, the solution $\theta$ of Eq.~\eqref{eq:irr} can be given by the following system of geometric equations:
\begin{align}\label{eqs:geo}
\left\{
\begin{aligned}
\alpha &= \arctan(\frac{y}{x}) \\
\beta &= \arccos(\frac{\sqrt{x^2 + y^2}}{2}) \\
\theta &= \alpha - \beta + 2m\pi , ~m \in \mathbb{Z}\\
\bar{\pi}\theta &= \alpha + \beta + 2k\pi,~k \in \mathbb{Z}  
\end{aligned}
\right.
\end{align}
\end{lemma}
We prove this lemma in Appendix~\ref{app:gemo}. 
As such, the task of finding the exact value of $\theta$ is transformed into: given $(x, y)$, determine the specific value of the integer $m$, such that there exists an integer $k$ satisfying the above system of equations.
By further simplifying the system of Eqs.~\eqref{eqs:geo}, our main objective becomes solving the following equation: 
\begin{align}
\{m\bar{\pi}\} &= \{\Omega\} \label{eq:solve},  \\
\Omega &= \frac{\alpha(1 - \bar{\pi}) + \beta(1 + \bar{\pi})}{2 \pi} \label{eq:omega},
\end{align}
where $\{\cdot\}$ returns the fractional part of a input number.
Here, we can first compute the exact value of $\Omega$ with linear complexity.
However, directly solving Eq.~\eqref{eq:solve} is very difficult. Intuitively, we can enumerate all possible values of $m$ and then search for the longest match with the fractional part of the target $\Omega$.
Nevertheless, the time complexity of the search process is still extremely high and not practical.
As our solution, we selectively mask certain digits of $\pi$.
Specifically, the value of $\bar{\pi}$ in our approach is constructed as:
\begin{align}\label{eq:ppi}
\bar{\pi} = 0.\underbrace{000\cdots}_{\lambda-1~\text{times}}1\underbrace{0000\cdots}_{\lambda~\text{times}}\underbrace{3589793238\cdots}_{\text{same as}~\pi}, 
\end{align}
where $\lambda$ is a quantization hyper-parameter.
We can guarantee that $\bar{\pi}$ is still an irrational number, as it can be obtained by adding $\pi$ to some rational numbers.
In this way, we can approximate $m$ to the first $\lambda$ decimal places of $\Omega$:
\begin{align}\label{eq:m}
m \approx \lfloor\{\Omega\} \times 10^{\lambda}\rfloor
\end{align}
where $\lfloor \cdot \rfloor$ represents the floor function.
We give an intuitive explanation of our construction in Appendix~\ref{app:exp}.
This solution guarantees that the error does not exceed $10^{-\lambda}$.
The entire process maintains linear complexity.

\subsection{Optimization with Rotation Quantization}
In the previous section, we discussed how to compress a pair of real-value parameters using a single rotation variable $\theta$.
To further reduce the memory usage, we aim to use a low-precise representation for quantizing the corresponding $\theta$.
Next, we will introduce the details of our framework.

\paratitle{Quantize the Rotation Angles.}
As introduced in Section~\ref{sec:irr}, our approach use the complex rotation scheme 
to represent parameter pairs. One limitation of this method is that it requires the magnitudes of  $x$ and $y$ (\ie $||x^2 + y^2||$) to be no more than 4. 
To address this, we first scale and normalize the parameters to ensure they lie within the representable range. Formally, given the split input matrices $(\bm{X}, \bm{Y}) = \text{Split}(\mathbf{T})$, the scaled tensor can be computed by a simple method:

\begin{align}
\Tilde{\bm{X}} &= \bm{X} / {w},
~\Tilde{\bm{Y}} = \bm{Y} / {w} \label{eq:scale}
\end{align}
where $w$ is the maximum absolute value of the numbers in both $\bm{X}$ and $\bm{Y}$.
Using this method, the value ranges of $\Tilde{\bm{X}}$ and $\Tilde{\bm{Y}}$ are both transformed to $[-1,1]$, satisfying the mapping conditions in Theorem~\ref{thm:irr}.
Note that we can also constrain them to a larger range (\eg $[-\sqrt{2}, \sqrt{2}]$) by multiplying $w$ by a coefficient.
Empirically, this setting has already achieved a good trade-off in quantization deviation, which we will discuss in Section~\ref{sec:dis}.
Afterwards, we compress the $\Tilde{\bm{X}}$ and $\Tilde{\bm{Y}}$ to a unique $\bm \Theta$, following the method introduced in Section~\ref{sec:irr}.
For further quantizing $\theta$,  we retain $\lambda$ decimal places of $\alpha - \beta$ (See Eq.~\eqref{eqs:geo}), with $10^{-\lambda}$ as the interval:

\begin{align}~\label{eq:qua}
\begin{aligned}
\Tilde{\theta} = \frac{\theta}{2\pi 10^{-\lambda}} =  m \cdot 10^{\lambda} + g,~~
g=\lfloor\frac{\alpha - \beta}{2\pi} \times 10^{\lambda}\rfloor
\end{aligned}
\end{align}

In this way, we can represent each $\theta$ with $2\lambda$ digits, where the first $\lambda$ digits correspond to the $m$ and the remaining $\lambda$ digits correspond to the $g$.
As such, we can set varied $\lambda$ to achieve different bit-widths.
Since the $\lambda$ is set based on the decimal system, on average, each parameter requires $\lambda \cdot \log_2 10 \approx 3.32 \lambda$ bits for storage.
In particular, the entire process remains linear complexity and supports parallel computation on GPUs, making it practical in the large scale model optimization. 
We present a detailed procedure for the quantization process in Algorithm~\ref{alg:qua}.

\begin{algorithm}[h]
    \caption{Quantization of the Rotation Angles}
    \label{alg:qua}
    \begin{algorithmic}[1]
        \STATE \textbf{Input:} Parameter $\textbf{T}$, size: $n$.
        
        \STATE Split $\textbf{T}$ into two tensors: $\bm {X}$ and $\bm {Y}$.
        \STATE Compute $w = \max(|{\bm X, \bm Y}|)$. 
        \STATE Scale $\bm X$ and $\bm Y$ according to Eq.~\eqref{eq:scale}.
        \STATE Compute $\bm \alpha$ and $\bm \beta$ using Eq.~\eqref{eqs:geo}.
        \STATE Compute $\bm \Omega$ according to Eq.~\eqref{eq:omega}.
        \STATE Compute $\bm m$ based on Eq.~\eqref{eq:m}.
        \STATE Calculate $\bm{\Theta}$ from Eq.~\eqref{eq:qua}.
        \STATE \textbf{Output:} Quantized parameter $\bm{\Theta}$, size: $n/2$; Scale Factor $w$, size: 1.
    \end{algorithmic}
\end{algorithm}

\revision{Through the rotation transformation (Eq.~\eqref{eq:irr}), the parameter scale is reduced by half. After quantization, the precision of the angle representation can be further reduced. 
In practice, we can compress the precision of an angle into an integer ranging from 0 to 99 (\ie $\lambda = 1$), achieving the same quantization effect as 3.32-bit (See Table~\ref{tab:ppl}).} 
Next we discuss how to apply our algorithm during the training process.

\paratitle{Optimization.}
\revision{The main advantage of our algorithm lies in its ability to reduce the optimization memory.
Since only a small part of parameters necessitate immediate computation during training~\cite{li2023memory}, we only restore this subset of parameters, while the remainder continues to reside in memory in their angular representation.} 
As shown in Figure~\ref{fig:model}, in our framework, the parameters stored in memory are $\bm{\Theta}$ and the corresponding $w$ (See Algorithm~\ref{alg:qua}).
When the parameters are involved in computing, we apply the inverse operations of the quantization process to activate them:
\begin{align}\label{eqs:inverse}
\left\{
\begin{aligned}
 \theta &= \Tilde{\theta} \times 2\pi 10^{-\lambda},\\
\Tilde{x} &= \cos\theta + \cos\bar{\pi}\theta,~\Tilde{y} = \sin\theta + \sin\bar{\pi}\theta,\\
x &= \Tilde{x} \cdot w,~~~~~y = \Tilde{y} \cdot w. 
\end{aligned}
\right.
\end{align}
In this way, we can easily recover the parameter tensor  $\bm{X}, \bm{Y}$ before quantization, and we then concatenate them to reconstruct the original $\textbf{T} = \textrm{Concat}(\bm{X}, \bm{Y})$.
\revision{Similar to existing quantization works~\cite{li2023memory, dettmers20218}, the activated parameters maintain high precision to ensure the accuracy of the gradients and predictions. }
Since only a small subset of parameters is involved in computation most of the time, the temporary storage overhead incurred by these high-precision activated parameters is negligible.
In most scenarios, the momentum of the optimizer accounts for the majority of the training parameters (\eg the momentum in the Adam constitutes 66.67\%). The primary application of our method is compressing this portion of the parameters, which significantly reduces memory usage. The detailed algorithm is presented in Algorithm~\ref{alg:adam}. 

\begin{algorithm}[h]
\caption{Adam with $\pi$-Quant (Differences from the original Adam are highlighted in \textcolor{teal}{teal}).}
\label{alg:adam}
\begin{algorithmic}[1]
\STATE {\bfseries Input:} Learning rate $\alpha$, decay rates $\beta_1$, $\beta_2$, $\epsilon$ (small constant for numerical stability)
\STATE {\bfseries Initialize:} \textcolor{teal}{$\bm m_0 = \textrm{Quant}(\bm 0)$}, \textcolor{teal}{$\bm v_0 = \textrm{Quant}(\bm 0)$}, $t = 0$ \COMMENT {\textcolor{teal}{using Algorithm~\ref{alg:qua}}}
\FOR{each iteration $t = 1, 2, \dots, T$}
    \STATE $t \gets t + 1$
    \STATE Compute the gradient $\nabla \bm \theta_t$
    \STATE \textcolor{teal}{$\bm m_{t-1} \gets \textrm{Restore}(\bm m_{t - 1})$} \COMMENT {\textcolor{teal}{according to Eq.~\eqref{eqs:inverse}}}
    \STATE $\bm m_t \gets \beta_1 \bm m_{t-1} + (1 - \beta_1) \nabla\bm \theta_t$
    \STATE \textcolor{teal}{$\bm v_{t-1} \gets \textrm{Restore}(\bm v_{t-1})$} \COMMENT {\textcolor{teal}{according to Eq.~\eqref{eqs:inverse}}}
    \STATE $\bm v_t \gets \beta_2 \bm v_{t-1} + (1 - \beta_2) (\nabla \bm\theta_t)^2$
    \STATE $\hat{\bm m}_t \gets \frac{\bm m_t}{1 - \beta_1^t}$ \COMMENT{Bias correction for $\bm m_t$}
    \STATE $\hat{\bm v}_t \gets \frac{\bm v_t}{1 - \beta_2^t}$ \COMMENT{Bias correction for $\bm v_t$}
    \STATE $\bm \theta_t \gets \bm \theta_{t-1} - \alpha \frac{\hat{\bm m}_t}{\sqrt{\hat{\bm v}_t} + \epsilon}$ \COMMENT{Parameter update}
    \STATE \textcolor{teal}{$\bm m_{t} \gets \textrm{Quant}(\bm m_{t})$} \COMMENT {\textcolor{teal}{according to Algorithm~\ref{alg:qua}}}
    \STATE \textcolor{teal}{$\bm v_{t} \gets \textrm{Quant}(\bm v_{t})$} \COMMENT {\textcolor{teal}{according to Algorithm~\ref{alg:qua}}}
\ENDFOR
\STATE {\bfseries Output:} Optimized parameters $\bm \theta_T$
\end{algorithmic}
\end{algorithm}

Like most training-oriented quantization methods~\cite{dettmers20218, li2023memory}, $\pi$-Quant introduces additional ``Quant'' and ``Restore'' computations for each iteration, which results in higher latency compared to the original Adam.
Specifically, given the parameter size $n$, the additional complexity of $\pi$-Quant is $O(n)$ per iteration, which can be negligible compared to the overall time complexity.

\subsection{Discussion}\label{sec:dis}
In this part, we discuss the advantage of $\pi$-Quant and compare our work with existing work.

\paratitle{Model Merits.} Our method has the following merits:

$\bullet$ \emph{Quantization Precision}.
We offer in-depth analyses of the quantization accuracy in $\pi$-Quant, and demonstrate its superiority over traditional approaches.
Specifically, the main deviation of our method comes from the approximation of the $m$ and $\alpha - \beta$ in Eq.~\eqref{eq:m} and Eq.~\eqref{eq:qua}, respectively.
Formally, given the deviation of $\Delta \theta$ in Eq.~\eqref{eq:qua}, the upper bound of quantization error for a specific point $(x, y)$ is given by:
\begin{align}~\label{eq:error}
\begin{aligned}
    ||\Delta x|| &< \Delta\theta \Big|\Big| (1 - \bar{\pi})\sin\theta' + \bar{\pi}y'||  < \Delta\theta, \\
    ||\Delta y|| &<  \Delta\theta\Big|\Big| (1 - \bar{\pi})\cos\theta' + \bar{\pi}x'\Big|\Big| < \Delta\theta,
\end{aligned}
\end{align}
where $\theta' \in [\theta, \theta + \Delta \theta]$.
We give a proof in Appendix~\ref{app:pro}.
It indicates that our method can reduce the quantization error in the corresponding angle $\theta$.
Further, given the precision width $\lambda$, the average quantization error can be given by: 
\begin{align}~\label{eq:avg}
\begin{aligned}
   \mathbb{E}(\Delta x) = \mathbb{E}(\Delta y) = O(\frac{2 \cdot (1 + \bar{\pi}) \cdot 10^{-\lambda}}{\pi}).
\end{aligned}
\end{align}
The proof is in Appedix~\ref{app:avg}.
This indicates that the quantization error increases as the bit width decreases.
Note that $2 \cdot (1 + \bar{\pi}) < \pi$, which further demonstrates that our method has a lower average quantization error compared to traditional methods (\ie $10^{-\lambda}$).
In practice, we reduce the bit-width of optimizer states down to 3.32-bit (\ie $\lambda = 1$), while maintaining full accuracy (See Section~\ref{sec:lang}).



$\bullet$ \emph{Non-uniformity}. 
Non-uniformity is an essential property that ensures the accuracy of quantization for optimization, as the parameters of neural networks typically exhibit a non-uniform distribution (\eg Gaussian distribution). 
In Figure~\ref{fig:err}, we plot the quantization error distribution between traditional quantization methods and our approach.
As shown, the errors of traditional methods are uniformly distributed across the two-dimensional  plane. In contrast, our method produces smaller errors near the zero point (the center of the heatmap), with larger errors at the periphery.
Additionally, we plot the accuracy distribution for both traditional methods and our approach. It is evident that our method allocates more precision around the zero point to better capture the parameters, while assigning less precision to the range of larger values. Since most parameters and momentum in large models follow a standard distribution with a mean close to zero, our method is more effective at quantizing these parameters.

\paratitle{Novelty and Differences.} 
In Table~\ref{tab:cmp}, we compare our method with existing optimizer state quantization methods.
To the best of knowledge, it is the first attempt that leverages the properties of irrational numbers for data compression, grounded in a new mathematical formula.
This approach transforms a complex-valued parameter into a single rotation angle, which can precisely halve the parameter scale of a wide range of modules such as optimizer states and model weights.
By the proposed system of geometric equations, our approach only requires linear-complexity calculations, possessing superior theoretical properties (\eg non-uniformity), with no need of complicated search process.
In contrast to traditional quantization methods~\cite{li2023memory, dettmers20218}, which perform numerical compression for each single element, we propose to conduct the quantization for each pair unit of parameters, by reducing the precision of the rotation angles.
As such, it provides a way to better capture the correlations between different elements, making it possible to further reduce precision to lower bit-widths.
In practice, $\pi$-Quant can reduce the precision of optimizer states to 3.32-bit, which can be flexibly extended into variable bit-widths, with minimal impact on the training speed (See Table~\ref{tab:ppl}).
In general, our method provides a precise and theoretically-grounded solution for compressing the parameter scale in the training of large-scale models.


\begin{figure}[ht]
\vskip 0.2in
\begin{center}
\centerline{\includegraphics[width=0.9\columnwidth]{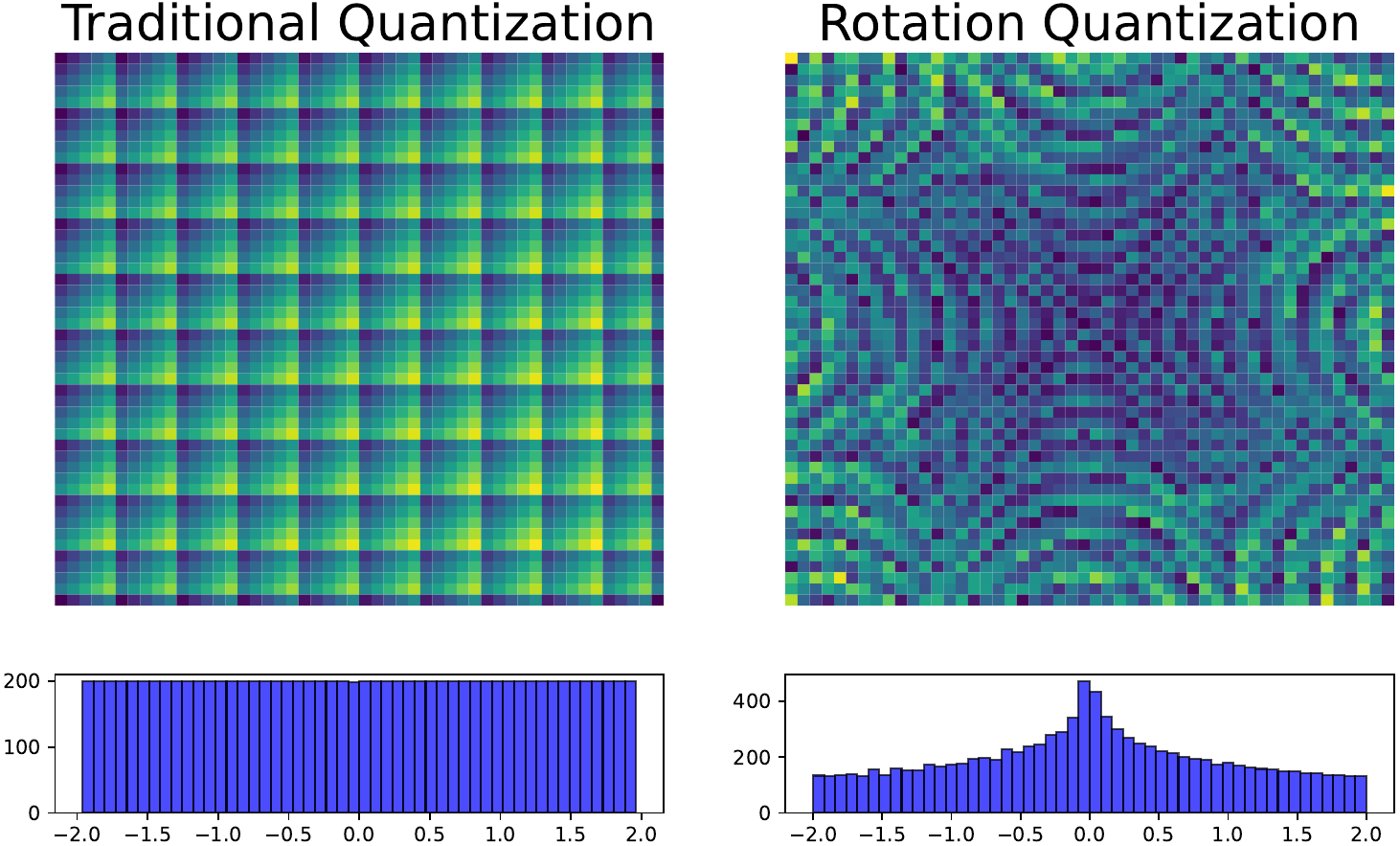}}
\caption{The error distribution (top) and precision distribution (bottom) between our method and traditional methods.}
\label{fig:err}
\end{center}
\vspace{-1em}
\end{figure}

\begin{table}[!ht]
\captionsetup{font={small}}
\small
\caption{Comparison of different quantization methods.}
\label{tab:cmp}
\resizebox{\columnwidth}{!}{
\begin{tabular}{@{}lccccc@{}}
\toprule
Methods & Bit-width & Search & Distribution  & Unit \\ \midrule
LinearQuant  & Flexible
 & No & Uniform & Single \\
Bnb.Adam    & 8 & Yes & Non-Uniform & Single \\ 
Lpmm.Adam & 4 & Yes & Non-Uniform & Single\\ 
$\pi$-Quant (ours) & 3.32~(Flexible)
 & No & Non-Uniform & Pair \\  \bottomrule
\end{tabular}
}
\end{table}




%% file: sec-exp.tex
\section{Experiment}
\label{sec:exp}

\subsection{Language Modeling Evaluation Under Different Bit-Widths}
\label{sec:lang}
\paratitle{Experimental Settings.}
We conduct a comprehensive comparison of our approach with state-of-the-art quantization methods, including Linear-Quant~\cite{vanhoucke2011improving} with different bit-widths, 8bit-Adam (Bnb)~\cite{dettmers20158}, 4bit-Adam (Lpmm)~\cite{li2023memory}, and standard FP32 precision. 
Following the previous work~\cite{peng2023yarn}, we continually pre-train the TinyLlama-1.1B checkpoint~\cite{zhang2024tinyllama} for 400 steps on the  PG-19~\cite{rae-iclr-2020-compressive} dataset, chunked into 64k segments bookended with the BOS and EOS token, with a context window of $2048$.
We report the test perplexity in the Proof-pile dataset, and evaluate the trained LLM across four public benchmarks from Huggingface~\cite{eval-harness}, including ARC-Challenge, Hellaswag, Lambada and PIQA.

\begin{table*}[t]
\caption{Validation of effectiveness under different bandwidths, where the bandwidth is enclosed in parentheses``()''. The task is language modeling in PG-19 dataset. The backbone large language model is TinyLLama. ``N/A'' denotes that ``nan loss'' (training collapse).}
\label{tab:ppl}
\centering
\resizebox{1.0\linewidth}{!}{
\begin{tabular}{l|c|c|cc|ccccc}
\toprule
Type                      & Approach                           &  Memory / Step Time         & Train Loss  $\downarrow$           & Test PPL $\downarrow$ & ARC-c $\uparrow$ & Hellaswag $\uparrow$ & Lambada $\uparrow$ & PIQA $\uparrow$ & Avg. \\\midrule
Full   &  FP32 (\textbf{32}) & 19.47 GB / 9.77 s & 2.423                       &  5.17  &  32.51  & 58.65   & \underline{56.55} & 73.18 & 55.22   \\
\cmidrule(rr){1-10}
\multirow{3}{*}{High}   & FP16 (\textbf{16})           &          15.53 GB / 8.93 s          & $>10^4$                       & 105.36   &  24.57  & 24.52   &   17.19 & 49.02 & 28.83
                           \\\cmidrule(rr){2-10}
                           & Linear-Quant (\textbf{16})     & 15.53 GB / 9.79 s                               & N/A                        & N/A   &  N/A  & N/A   &  N/A & N/A & N/A 
                           \\\cmidrule(rr){2-10}
                           
                           & $\pi$-Quant (\textbf{13.28})  & 15.40 GB / 11.95 s  & \textbf{2.420}                               & \underline{5.15}                        & \underline{33.02}   &  \textbf{58.78}  & \textbf{56.76}   &   \underline{73.23} & \textbf{55.45}
                           \\\midrule
 \multirow{3}{*}{Medium}                          & Linear-Quant (\textbf{8})                                              & 13.44 GB / 14.97 s     & N/A                     &  N/A  &  N/A  & N/A   & N/A & N/A  & N/A \\\cmidrule(rr){2-10}
                                                & Bnb.Adam (\textbf{8})                                             & 13.44 GB / 10.47 s     & \underline{2.421}                     &  \textbf{5.14} &  \underline{33.02}  & 58.64   & 56.51 & 73.19  & 55.34 \\\cmidrule(rr){2-10}
                                                & $\pi$-Quant (\textbf{6.64})                                            & 13.58 GB / 11.70 s  & \underline{2.421}                               & 5.16                        & \underline{33.02}   &  \underline{58.75}  & \textbf{56.76}   &   73.06 & \underline{55.40}\\\midrule

\multirow{2}{*}{Low}   & Lpmm.Adam (\textbf{4})                                        &      12.30 GB / 10.66 s                   &  2.422  &  5.41  &  31.23  & 58.47  & 55.19 & 72.85 & 54.44 \\
\cmidrule(rr){2-10}
                           & {$\pi$-Quant} (\textbf{3.32}) & 11.32 GB / 10.56 s          &   2.422 &  \textbf{5.14}  &  \textbf{33.53}  & 58.45 & 56.05 & \textbf{73.34} & 55.34\\

\bottomrule
\end{tabular}}
\end{table*}


\paratitle{Results}. As shown in Table~\ref{tab:ppl}, we can observe that Linear-Quant~\cite{vanhoucke2011improving} exhibits the poorest performance among all compared methods, with its training process being highly unstable and frequently resulting in catastrophic failures.
Besides, Bnb~\cite{dettmers20218} exhibits remarkable robustness, achieving performance metrics on par with the full-precision baseline.
However, Lpmm Adam~\cite{li2023memory} shows a notable degradation in performance metrics, even though its training loss is not much different from full precision.
Notably, $\pi$-Quant achieves lossless performance in both 13.28-bit and 6.64-bit settings, with metrics even surpassing those of the full-precision baseline.
Remarkably, our method reduces state bit-widths to 3.32 bits while maintaining similar convergence trend (Figure~\ref{fig:loss}) and performance metrics.
These results show that our rotation-based quantization method can effectively reduce the quantization loss, serving as a robust replacement for full-precision optimizers.

Besides, we report the training memory usage and the optimization step time in Table~\ref{tab:ppl}.
As for training memory, we can observe that Bnb achieves a 30.8\% reduction in memory usage compared to the FP32 Adam; Lpmm can further reduce memory usage due to its lower precision representation. 
In comparison, our method consumes the least memory, achieving a 41.9\% reduction in memory usage.
As for training speed, quantization-based methods introduce additional quantization and de-quantization operations during training, leading to relatively higher latency.
Notably, Bnb and Lpmm Adam have compiled GPU binary search kernels, achieving comparable quantization latency.
Since our method requires calculating parameter angles during Quant and Restore stages, resulting in additional computational overhead, which is a limitation of our approach.

\begin{figure}[h]
\vskip 0.2in
\begin{center}
\centerline{\includegraphics[width=.85\columnwidth]{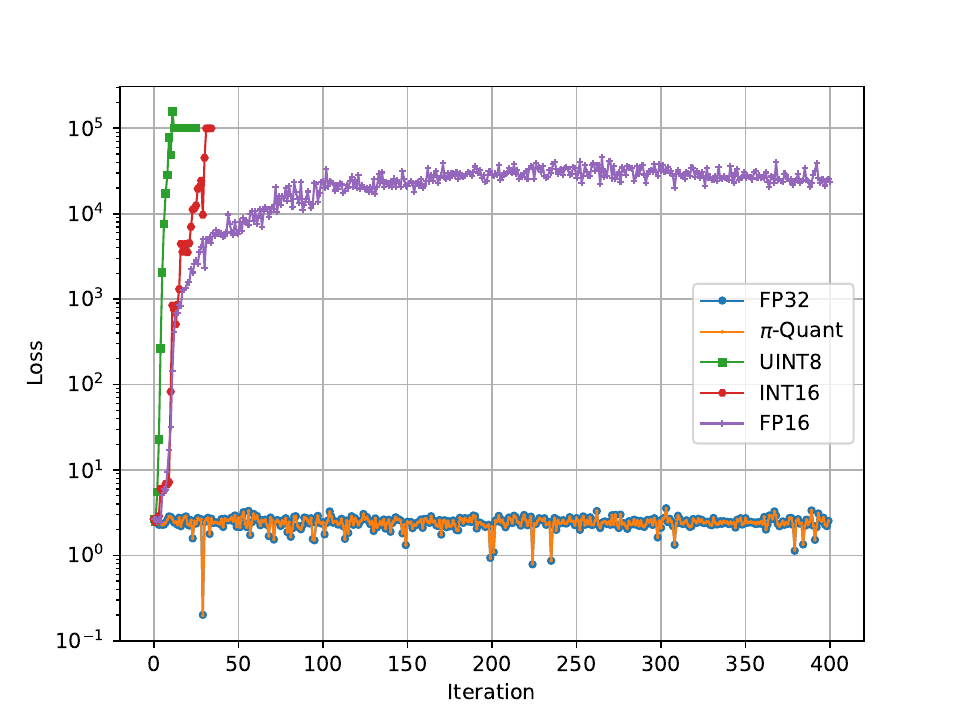}}
\caption{Loss comparison between FP32 and our approach.}
\label{fig:loss}
\end{center}
\vskip -0.2in
\end{figure}

\subsection{Experimental Analysis}

\subsubsection{Accuracy on Other Downstream Tasks}

\paratitle{Experimental Settings.} We evaluate $\pi$-Quant on five classification tasks on long-range-arena  benchmark~\cite{tay2020longrangearenabenchmark}, including Listops~\cite{nangia2018listopsdiagnosticdatasetlatent}, Text classification on IMDb review dataset~\cite{maas-etal-2011-learning}, Document Retrieval on AAN dataset~\cite{radev-etal-2009-acl}, Pathﬁnder~\cite{linsley2019learninglongrangespatialdependencies}, and Image
classification on CIFAR-10~\cite{krizhevsky2009learning}.
Besides, we also introduce two sequence-to-sequence tasks, including text summarization on the Samsum~\cite{gliwa2019samsum} dataset and sequential recommendation on the Movielens~\cite{harper2015movielens} dataset.
Specifically, we use the standard transformer as the backbone, and the implementation details follow the work~\cite{chen2021skyformer}.
We report the classification accuracy for the long-range-arena tasks, and report the belu scores and recommendation metrics for the other tasks.

\paratitle{Results.} We can observe that Bnb~\cite{dettmers20158} shows significant accuracy drops on Image, Retrieval, Listops, Summary and RecSys tasks, but surpasses full-precision on Text and PathFinder tasks.
Besides, Lpmm~\cite{li2023memory} shows suboptimal performance across most tasks, with only comparable results on PathFinder..
These results suggest that the predefined quantization precision in these methods is not always optimal across different tasks.
In comparison, we observe that $\pi$-Quant outperforms the original Adam across all tasks, suggesting that the rotation quantization may represent a superior optimization setting.

\begin{table}[!t]
  \centering
  \small
  \captionsetup{font={small}}
  \caption{Performance comparison on other machine learning tasks.} 
  \label{tab:fis}
  \resizebox{.95\columnwidth}{!}{
  \begin{tabular}{c|c|cccc}
    \toprule
    \textbf{Task} & \textbf{Metric} & \textbf{Adam} & \textbf{Bnb} & \textbf{Lpmm} & \textbf{$\pi$-Quant} \\
    \hline \hline
    \multirow{2}{*}{{\tabincell{c}{Text}}} & Acc. & \underline{64.63} & \textbf{64.68} & 64.14 & \textbf{64.68}\\
    & Loss  & 0.634 & \textbf{0.630} & 0.635 & \underline{0.631}\\
    \hline
    \multirow{2}{*}{\tabincell{c}{Image}} & Acc. & \underline{40.13} & 39.71 & 39.92 & \textbf{41.22}\\
    & Loss  & 1.896 & \underline{1.763} & 1.771 & \textbf{1.745}\\
    \hline
    \multirow{2}{*}{\tabincell{c}{Retrieval}} & Acc. & \underline{80.70} & 79.92 & 79.19 & \textbf{80.79} \\
    & Loss  & \textbf{43.61} & 45.21 & 45.93 & \underline{44.32}\\
    \hline
    \multirow{2}{*}{\tabincell{c}{Listops}} & Acc. & \underline{38.21} & 37.30 & 37.30 & \textbf{38.51} \\
    & Loss  & 1.780 & \underline{1.715} & 1.741 & \textbf{1.646}\\
    \hline
    \multirow{2}{*}{\tabincell{c}{PathFinder}} & Acc. & 69.15 & 70.83 & \textbf{72.05} & \underline{71.92}\\
    & Loss  & 55.13 & \underline{53.91} & \textbf{49.91} & 58.57\\
    \hline
    \multirow{4}{*}{{\tabincell{c}{Text\\Summary}}} & Blue-1 & \textbf{14.42} & 13.34 & 13.71 & \underline{14.32}\\
    & Blue-2  & \underline{8.30} & 7.68 & 7.95 & \textbf{8.72}\\
    & Blue-3  & \underline{4.03} & 3.87 & 3.99 & \textbf{4.74}\\
    & Blue-4 & \underline{1.54}   & 1.61 & 1.57 & \textbf{2.02}\\
    \hline
    \multirow{4}{*}{{\tabincell{c}{RecSys}}} & Recall@10 & \underline{28.19} & 27.75 & 28.01 & \textbf{28.76}\\
    & Ndcg@10  & \underline{15.51} & 15.34 & 15.41 & \textbf{15.85}\\
    & Mrr@10  & \underline{11.66} & 11.58 & 11.53 & \textbf{11.94}\\
    & Hit@10 & \underline{28.19}   & 27.75 & 28.01 &  \textbf{28.76}\\
  \bottomrule
\end{tabular}}
\vspace{-1em}
\end{table}

\subsubsection{Ablation Study}
We conduct ablation studies to explore the impact of $\bar{\pi}$ settings in our approach. 
As introduced in Section~\ref{sec:irr}, the design of $\bar{\pi}$ influences the accuracy of solving for \( \theta \) using Eq.~\eqref{eq:m}.
To verify this conclusion, we change the value of $\bar{\pi}$ as two forms: (1) as a truncated rational number, \ie $\bar{\pi}_1 = 0.001$ and (2) setting its integer part to other values, \eg $\bar{\pi}_2 = 3.0010003589792...$.
Specifically, We focus on analyzing quantization errors for two distributions: Gaussian (\ie $\mathcal{N}(0, 1)$) and uniform distributions (\ie $\mathbb{U}(0, 1)$).
As shown in Figure~\ref{fig:alb}, we can see that our method achieves lower error in fitting Gaussian distributions, revealing its non-uniform fitting capability.
In addition, $\bar\pi$ has a lower error than $\bar\pi_1$ as its decimal portion provides error compensation.
Besides, we find that setting the integer part of $\bar\pi$ to zero yields minimal error..

\begin{figure}[ht]
\vskip 0.2in
\begin{center}
\centerline{\includegraphics[width=0.7\columnwidth]{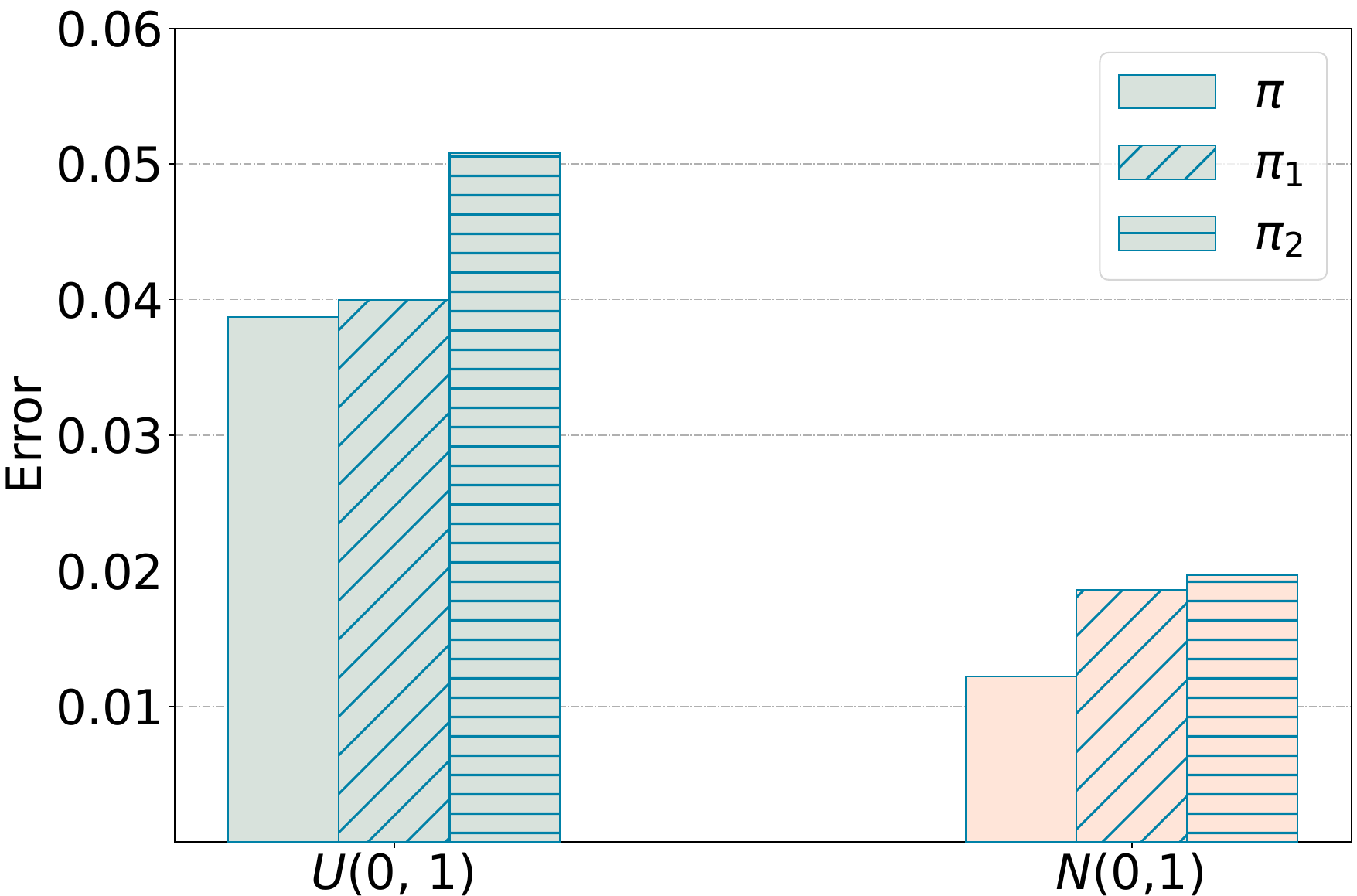}}
\caption{Ablation Study of the $\pi$ setting in our approach.}
\label{fig:alb}
\end{center}
\vspace{-1em}
\end{figure}

\subsubsection{Visualizing the Gradient Update Process}
To have an intuitive understanding of our approach, we visualize our approach's gradient descent through a simple task: minimizing the Himmelblau's function\footnote{$f(x, y) = (x^2 + y - 11)^2 + (x + y^2 - 7)^2$}.
As shown in Figure~\ref{fig:vis},  we establish multiple starting points and compared the training processes across different optimizers, including Adam with varying bit-widths, SGD, and our proposed method.
We can observe that Adam's optimization path is typically shorter than SGD's, demonstrating how Adam's momentum accelerates convergence.
Our method achieves the shortest paths across all initialization points, particularly at start point 4 where it successfully finds the correct minimum while Adam becomes trapped in a local saddle point.
These findings demonstrate that our rotation quantization method can enhance convergence speed in gradient descent and facilitates escape from local minima and saddle points.
\begin{figure}[h]
\vskip 0.2in
\begin{center}
\centerline{\includegraphics[width=.9\columnwidth]{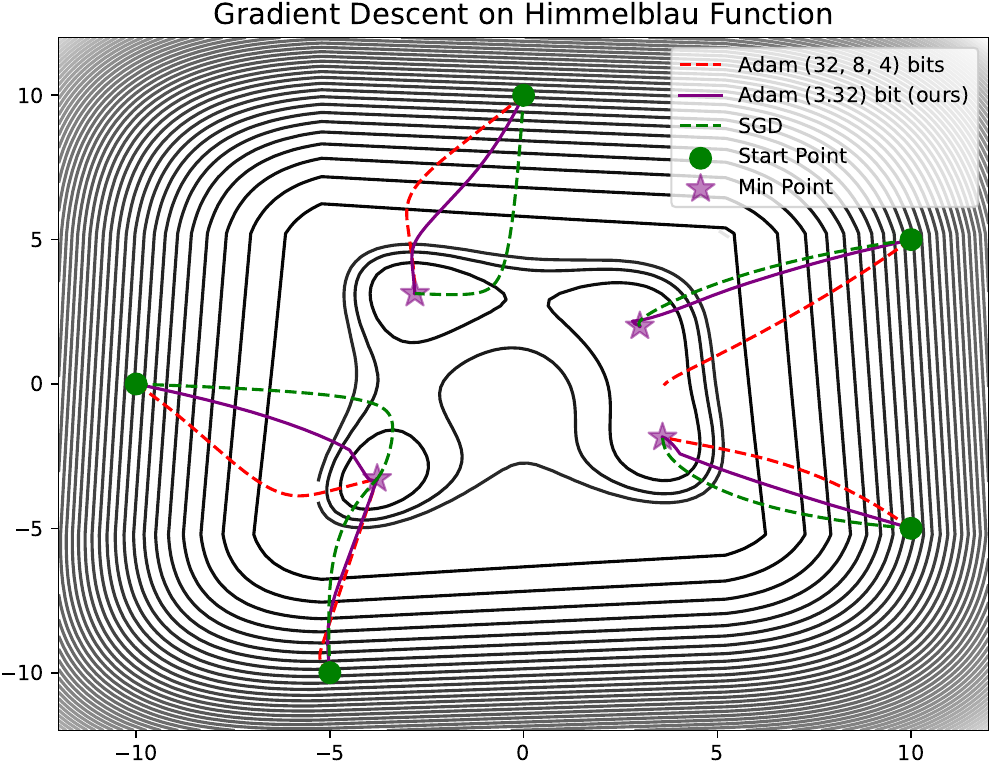}}
\caption{Visualizing the gradient descent process.}
\label{fig:vis}
\end{center}
\vspace{-1em}
\end{figure}

%% file: sec-related.tex
\section{Related Work}

\paratitle{Model Quantification.}
Large language models~\cite{zhao2023survey} have achieved outstanding performance across various tasks.
However, the extensive parameters in LLMs results in substantial memory usage, significantly increasing the cost of deploying them in real-world scenarios.
To overcome this challenge, a number of model quantization methods have been proposed for reducing the memory footprint of LLMs, by mapping floating-point numbers to integers~\cite{gholami2022survey}.
Typically, there are two classes of quantization approaches, namely \emph{training-oriented quantization} and post-training quantization (PTQ).
Among them, most PTQ work primarily focuses on separately handling  outliers, \ie extremely large values in activations~\cite{dettmers2022gpt3} or sensitive weights~\cite{lin2024awq, shang2023pb, dettmers2023spqr, guan2024aptq, lee2024owq}.
Specially, SmoothQuant~\cite{xiao2023smoothquant} propose a per-tensor quantization method for compressing the weights and activations of Transformers in a coarse-grained view, and ZeroQuant~\cite{yao2022zeroquant} improve it by a token-wise quantization approach in a fine-grained view.
Despite the progress, PTQ approaches often require additional data and fine-tuning for retaining the performance.
As such, a number of approaches~\cite{liu2023llm, wang2023bitnet} propose to conduct the quantization during training.
Through extensive experiments on LLaMA, their approach demonstrates promising results for weight and activation quantization, but remains unsuitable for momentum quantization.
To this end, Bnb~\cite{dettmers20218} and Lpmm~\cite{li2023memory} propose a search-based approach to quantize the momentum to the frequently occuring elements.
However, these pre-set elements cannot adapt well to all tasks, which may lead to degraded performance in certain tasks.

\paratitle{Optimization with Complex Rotations.}
Traditional neural networks are typically modeled in a real vector space, which exhibits certain constraints when conducting mathematical operations such as exponentiation and vector rotations.
To enhance the capability of representation learning, a number of approaches~\cite{tian2023eulernet, tian2024eulerformer, tian2024rotative, su2024roformer, tang2022image,li2023angle} have proposed conducting representation learning in complex vector space.
Typically, rotation techniques in complex space have been widely applied in the modeling of language models.
Specifically, RoPE~\cite{su2024roformer} employ complex rotation to inject the position information into the attention mechanism; EulerFormer~\cite{tian2024eulerformer} fully complexifies the Query and Key, enabling adaptive modeling of semantic and positional information; xPos~\cite{sun2022length} employs a two-dimensional pairwise rotation technique to enhance the positional embeddings in Transformers.
Unlike these work, we attempt to utilize complex rotations based on irrational numbers for quantization, allowing it to exhibit a non-uniform distribution for parameter optimization.

%% file: sec-conclusion.tex
\section{Conclusion}

In this paper, we proposed a novel optimizer state compression approach $\pi$-Quant.
Different from prior work, $\pi$-Quant compressed the parameters using a complex rotation scheme, by leveraging the properties of irrational numbers according to our proposed mathematical findings.
In $\pi$-Quant, the two dimensional parameter pairs were efficiently converted into a single rotation angle with our proposed system of geometric equations, which precisely halved the parameter scale with linear complexity.
Further, $\pi$-Quant introduced an effective quantization algorithm, which reduced the precision of the rotation angles, all while maintaining full accuracy.
In principle, $\pi$-Quant possessed lower quantization error and demonstrated a non-uniform precision distribution, which served as an effective replacement of full precision optimizers.
As the further work, we will consider reducing the quantization complexity, and testing the capacity of $\pi$-Quant in more compression scenarios, \eg KV cache compression in the large language models.

%% file: sec-appendix.tex
\newpage
\section{Proof of Theorem~\ref{thm:irr}.}
\label{sec:app1}
Note that the Eq.~\ref{eq:irr} can be written in the form of the following function $f: \theta \rightarrow (x, y)$, by using Euler's formula (\ie $e^{ix} = \cos x + i\sin x$):

\begin{align}\label{eqs:paraeq}
\left\{
\begin{aligned}
x &= \cos\theta + \cos \bar{\pi}\theta \\
y &= \sin\theta + \sin \bar{\pi}\theta \\ 
\end{aligned}
\right.
\end{align}

We first prove that the function of Eq.~\ref{eqs:paraeq} is non-periodic, which is equivalent to proving the following lemma:

\begin{lemma}
\label{lem:usefullemma}
There does not exist a period $T$ such that the following equation holds for all $\theta \in \mathbb{R}$:
\begin{align}\label{eqs:priod}
\left\{
\begin{aligned}
\cos\theta + \cos \bar{\pi}\theta &= \cos(\theta + T) + \cos (\bar{\pi}(\theta + T)) \\
\sin\theta + \sin \bar{\pi}\theta &= \sin(\theta + T) + \sin (\bar{\pi}(\theta + T)) \\ 
\end{aligned}
\right.
\end{align}
\end{lemma}

\begin{proof} 
By contradiction, assume that there exists a $T$ such that the above equation holds universally.
Differentiate both sides of Eq.~\eqref{eqs:priod}, \ie $\nabla f(\theta) = \nabla f(\theta + T)$:
\begin{align}\label{eqs:diff}
\left\{
\begin{aligned}
\sin\theta + \bar{\pi}\sin\bar{\pi}\theta &= \sin(\theta + T) + \bar{\pi}\sin (\bar{\pi}(\theta + T)) \\
\cos\theta + \bar{\pi}\cos\bar{\pi}\theta &= \cos(\theta + T) + \bar{\pi}\cos (\bar{\pi}(\theta + T)) \\ 
\end{aligned}
\right.
\end{align}
Subtract Eq.~\eqref{eqs:diff} from Eq.~\eqref{eqs:priod}:
\begin{align}\label{eqs:sub}
\left\{
\begin{aligned}
(\bar{\pi}-1)\sin\bar{\pi}\theta &= (\bar{\pi}-1)\sin (\bar{\pi}(\theta + T)) \\
(\bar{\pi}-1)\cos\bar{\pi}\theta &= (\bar{\pi}-1)\cos (\bar{\pi}(\theta + T)) \\ 
\end{aligned}
\right.
\end{align}
Eliminate the common terms:
\begin{align}\label{eqs:pit}
\left\{
\begin{aligned}
\sin\bar{\pi}\theta &= \sin (\bar{\pi}(\theta + T)) \\
\cos\bar{\pi}\theta &= \cos (\bar{\pi}(\theta + T)) \\ 
\end{aligned}
\right.
\end{align}
This indicates that $T$ must be an integer multiple of the smallest positive period of $\sin\bar{\pi}\theta$, \ie $T = 2p{\pi}/{\bar{\pi}}, p \in \mathbb{Z}$.
In additional, by subtracting Eq.~\eqref{eqs:pit} from Eq.~\eqref{eqs:priod}, we have:
\begin{align}\label{eqs:opit}
\left\{
\begin{aligned}
\sin\theta &= \sin (\theta + T) \\
\cos\theta &= \cos (\theta + T) \\ 
\end{aligned}
\right.
\end{align}
This indicates that $T$ must be an integer multiple of the smallest positive period of $\sin\theta$, \ie $T = 2q{\pi}, q \in \mathbb{Z}$.
Therefore, we have $T = 2q{\pi} = 2p{\pi}/{\bar{\pi}} \Longrightarrow {\bar{\pi}} = p / q$, which contradicts the fact that $\bar{\pi}$ is an irrational number.
Therefore, the original lemma holds.
\end{proof}
At the same time, it is easy to prove that the function of Eq.~\eqref{eqs:paraeq} is bounded: $|| x^2 + y^2|| = ||2 + 2\cos((1 - \bar{\pi})\theta)|| \leq 4$.
Therefore, the function's trajectory will traverse every point within the domain's circle, since it is non-closed and continuous.
In Figure~\ref{fig:vis}, we visualize the trajectory of Eq.~\eqref{eqs:paraeq} as $\theta$ varies in different ranges.
This reflects that when the range of $\theta$ is sufficiently large, any arbitrary two-dimensional point can be uniquely mapped to its corresponding $\theta$.

\begin{figure}[!h]
  \centering
  \captionsetup{font={small}}
  \subcaptionbox{$\theta \in [0, 100]$}{
    \includegraphics[width=0.3\linewidth]{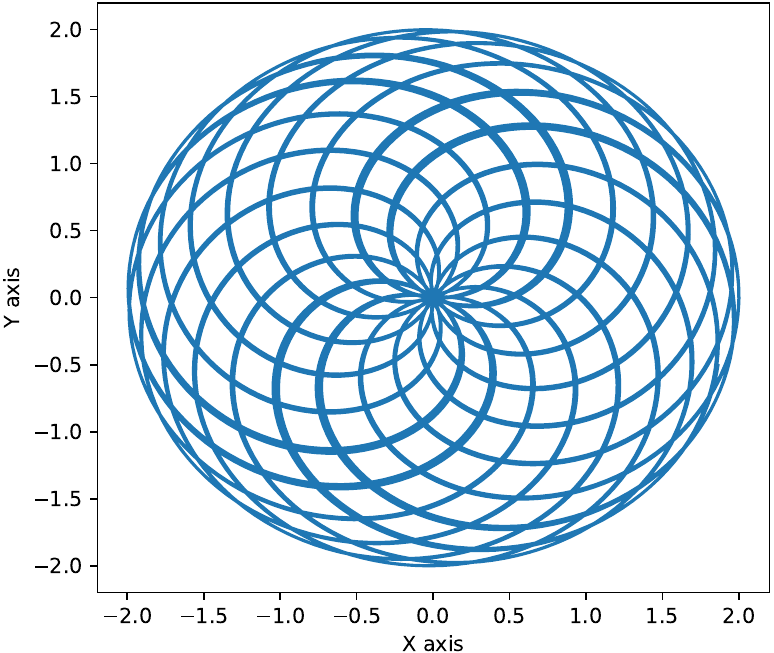}
  }
  \subcaptionbox{$\theta \in [0, 300]$}{
    \includegraphics[width=0.3\linewidth]{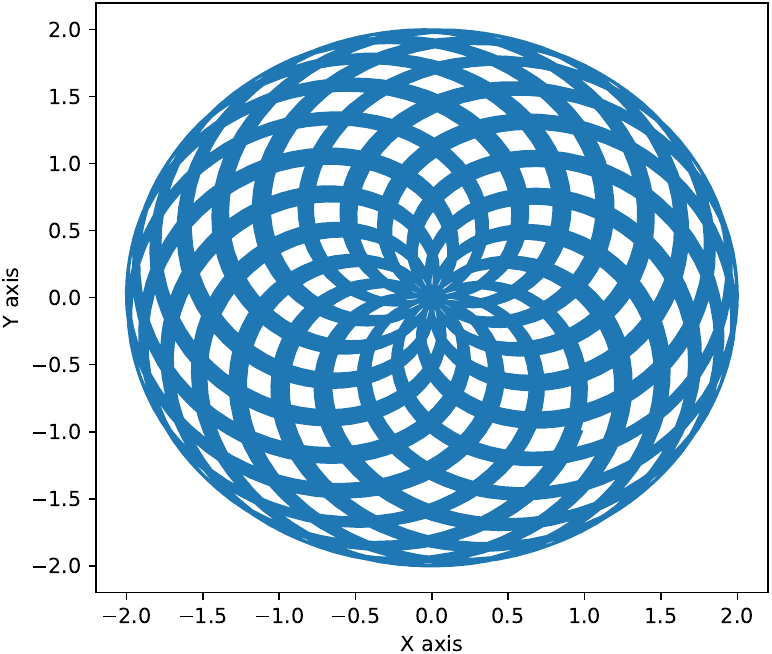}
  }
  \subcaptionbox{$\theta \in [0, 1000]$}{
    \includegraphics[width=0.3\linewidth]{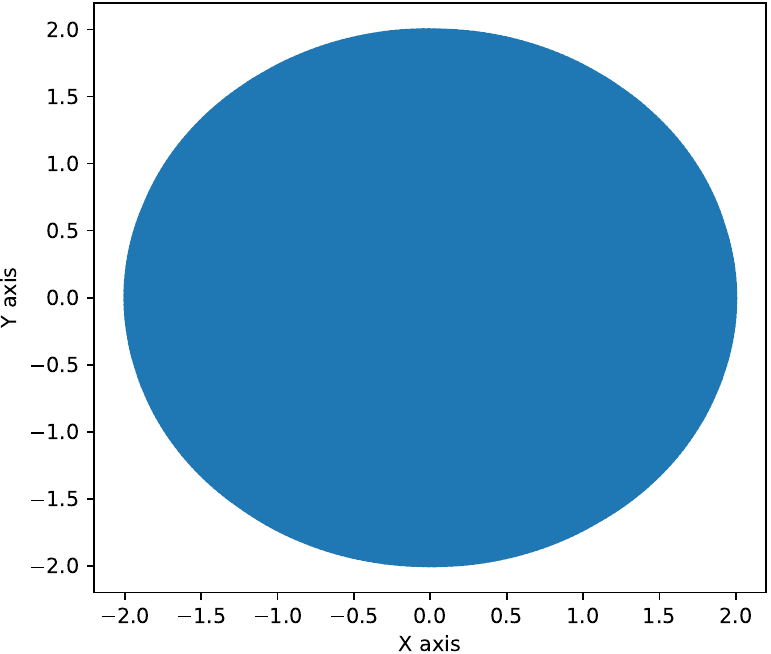}
  }
  \captionsetup{font={small}}
  \caption{Visualization of the trajectory of Function Eq.~\eqref{eqs:paraeq}.}
  \label{fig:vis}
\end{figure}

\section{Proof of Lemma~\ref{lem:geolemma}.}
\label{app:gemo}

\begin{figure}[ht]
\vskip 0.2in
\begin{center}
\centerline{\includegraphics[width=0.7\columnwidth]{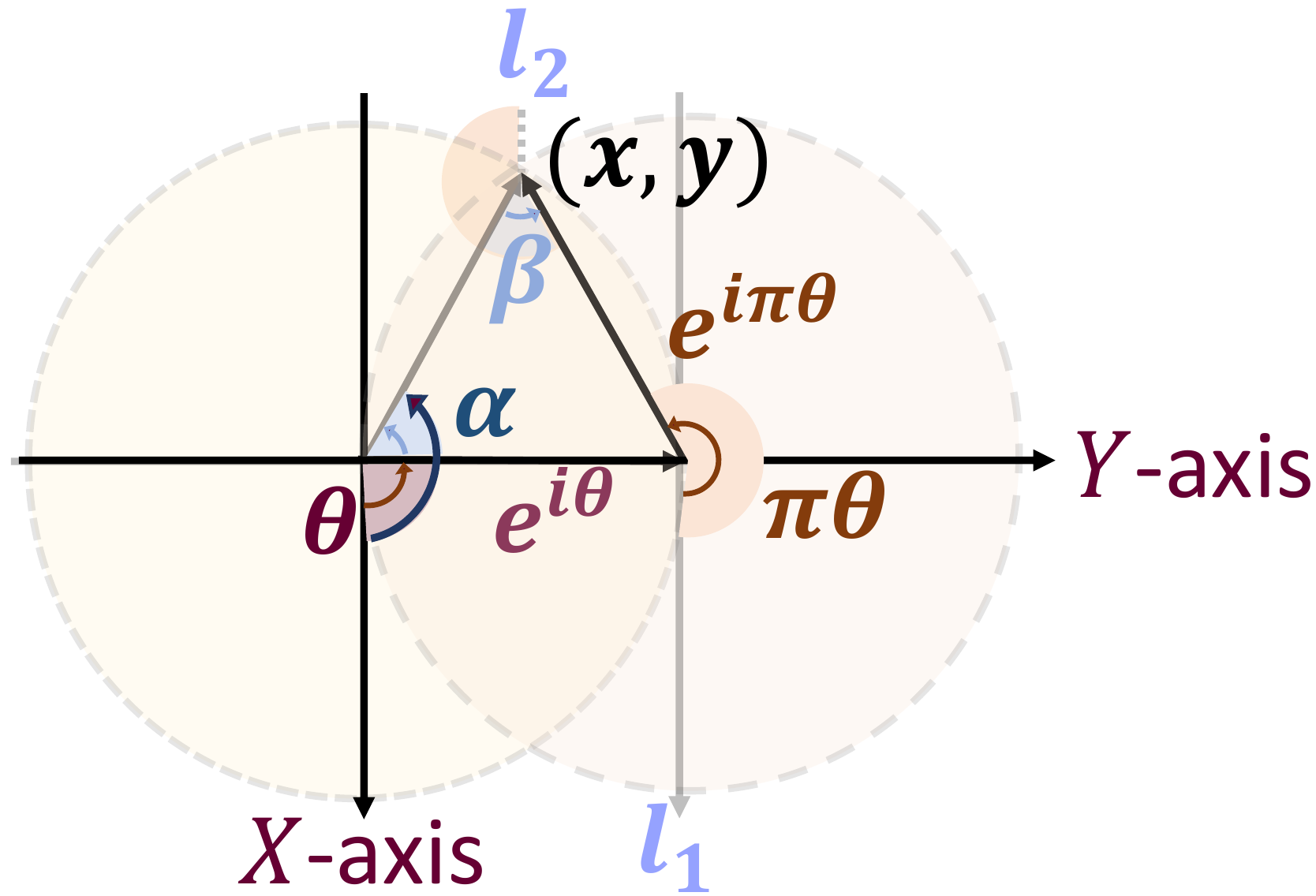}}
\caption{Geometric relationship in $x+iy = e^{\i\theta} + e^{i\pi\theta}$}
\label{fig:gemo}
\end{center}
\vskip -0.2in
\end{figure}

As shown in Figure~\ref{fig:gemo}, according to the definition of tangent function, we have:
\begin{align}\label{eq:11}
\alpha = \mathrm{arctan}(\frac{y}{x}).
\end{align}

Furthermore, since the magnitudes of $e^{i\theta}$ and $e^{i\pi\theta}$ are both 1, according to the cosine theorem, we have:

\begin{align}\label{eq:12}
\beta = \arccos(\frac{x^2 + y^2 + 1 - 1}{2 \cdot \sqrt{x^2 + y^2} \cdot 1}) = \arccos(\frac{\sqrt{x^2 + y^2}}{2}).
\end{align}

In particular, we also know that the triangle formed by $e^{i\theta}$, $e^{i\pi\theta}$, and $Oxy$ is an isosceles triangle. According to the angular relationship, we can easily deduce that:

\begin{align}\label{eq:13}
\alpha - \beta = \theta + 2m\pi,~~ m \in \mathbb{Z}.
\end{align}

Additionally, since $l_1$ and $l_2$ are parallel, we have:

\begin{align}\label{eq:14}
\alpha + \beta = \pi\theta + 2k\pi,~~ k \in \mathbb{Z}.
\end{align}

Therefore, Lemma~\ref{lem:geolemma} is proved by combing Eq.~\eqref{eq:11} ~ Eq.~\eqref{eq:14}.

\section{Proof of Eq.~\eqref{eq:error}.}\label{app:pro}

Recall that $x = \cos \theta + \cos \bar{\pi} \theta$, and we have:

\begin{align}\label{eqs:pp}
\nonumber
\begin{aligned}
||\Delta x|| &= || \cos(\theta + \Delta \theta) - \cos \theta +  \cos(\bar{\pi} (\theta + \Delta \theta)) -  \cos \bar{\pi} \theta|| \\
&= \Big|\Big|\int_{\theta}^{\theta + \Delta\theta}(\sin\theta + \bar{\pi}\sin\bar{\pi}\theta) d\theta\Big|\Big| \\
&\leq \int_{\theta}^{\theta + \Delta\theta}\Big|\Big|(\sin\theta + \bar{\pi}\sin\bar{\pi}\theta) \Big|\Big|d\theta
\end{aligned}
\end{align}

Let $\theta' = \textrm{argmax}_{\theta}{\Big|\Big|(\sin\theta + \bar{\pi}\sin\bar{\pi}\theta) \Big|\Big|}$, we have:

\begin{align}
\begin{aligned}
\Delta x &< \Delta \theta \Big|\Big|\sin\theta' + \bar{\pi}\sin\bar{\pi}\theta' \Big|\Big| \\
&=\Delta \theta\Big|\Big|(1 - \bar{\pi})\sin\theta' + \bar{\pi}(\sin\theta' + \sin\bar{\pi}\theta') \Big|\Big|\\
&= \Delta\theta\Big|\Big|(1 - \bar{\pi})\sin\theta' + \bar{\pi}y'\Big|\Big|
\end{aligned}
\end{align}

Since $|| y' || < 1$ (See Eq.~\eqref{eq:scale}), thus we have:
\begin{align}
\begin{aligned}
\Delta x &< \Delta \theta \Big|\Big|(1 - \bar{\pi}) + \bar{\pi} \Big|\Big| = \Delta\theta
\end{aligned}
\end{align}

Similarly, we can prove that $\Delta y < \Delta\theta\Big|\Big| (1 - \bar{\pi})\cos\theta' + \bar{\pi}x'\Big|\Big| < \Delta \theta$.

\section{Proof of Eq.~\eqref{eq:avg}.}\label{app:avg}
Note that the angle space is $[0, 2\pi10^{\lambda}]$, remark $M = 2\pi10^{\lambda}$, and we have:

\begin{align}
\begin{aligned}
\mathbb{E}(\Delta x) &= \frac{1}{M}\int_{0}^{M}{10^{-\lambda}} ||\sin\theta + \bar\pi\sin\bar\pi \theta|| d\theta \\
&< \frac{10^{-\lambda}}{M} (\int_{0}^{M}||\sin\theta||d\theta + \int_{0}^{M}\bar\pi||\sin\bar\pi\theta||d\theta)\\
&= \frac{10^{-\lambda}}{M} (\frac{2\cdot M}{\pi} + \frac{2\cdot M \cdot \bar\pi}{\pi})\\
&= \frac{2 \cdot (1 + \bar{\pi}) \cdot 10^{-\lambda}}{\pi}
\end{aligned}
\end{align}

Similarly, we have:

\begin{align}
\begin{aligned}
\mathbb{E}(\Delta y) &= \frac{1}{M}\int_{0}^{M}{10^{-\lambda}} ||\cos\theta + \bar\pi\cos\bar\pi \theta|| d\theta \\
&< \frac{10^{-\lambda}}{M} (\int_{0}^{M}||\cos\theta||d\theta + \int_{0}^{M}\bar\pi||\cos\bar\pi\theta||d\theta)\\
&= \frac{10^{-\lambda}}{M} (\frac{2\cdot M}{\pi} + \frac{2\cdot M \cdot \bar\pi}{\pi})\\
&= \frac{2 \cdot (1 + \bar{\pi}) \cdot 10^{-\lambda}}{\pi}
\end{aligned}
\end{align}

\section{Explanation of Eq.~\eqref{eq:m}.}\label{app:exp}

Note that $\bar{\pi}$ can be written as: $\bar{\pi} = 10^{-\lambda} + 10^{-2\lambda} \cdot 0.3589793238$:
\begin{align}\label{eq:1}
\bar{\pi} = 0.\underbrace{000\cdots1}_{10^{-\lambda}}\underbrace{0000\cdots}_{10^{-2\lambda}}\underbrace{3589793238\cdots}_{\text{same as}~\pi}.
\end{align}

Since $m$ retains the first $\lambda$ decimal places of $\Omega$, we have: 

\begin{align}\label{eq:2}
m \cdot \bar{\pi} = \underbrace{m \cdot 10^{-\lambda}}_{\text{first $\lambda$ decimal of $\Omega$}} + \underbrace{m \cdot 10^{-2\lambda} \cdot 0.3589793238}_{\text{compensation part}}.
\end{align}

In this setting, we can guarantee that the difference between $m\cdot \bar{\pi}$ and $\Omega$ does not exceed $10^{-\lambda}$.
To further understand the construction of $m$, we present an illustrative example with $\Omega = 1.97525751858\cdots$.
If we set $\lambda = 4$, then $m = \lfloor\{\Omega\} \times 10^{4}\rfloor = 9752$ (first 4 decimal of $\Omega$).
We have:

\begin{align}\label{eq:3}
m \cdot \bar{\pi} = 0.\underbrace{9752}_{m \cdot 10^{-\lambda}}\underbrace{350076636569}_{\text{compensation part}} \approx \{\Omega\}.
\end{align}

The error is $0.000022510916...$.

%% file: example_paper.bbl
\begin{thebibliography}{40}
\providecommand{\natexlab}[1]{#1}
\providecommand{\url}[1]{\texttt{#1}}
\expandafter\ifx\csname urlstyle\endcsname\relax
  \providecommand{\doi}[1]{doi: #1}\else
  \providecommand{\doi}{doi: \begingroup \urlstyle{rm}\Url}\fi

\bibitem[Abadi et~al.(2016)Abadi, Barham, Chen, Chen, Davis, Dean, Devin, Ghemawat, Irving, Isard, et~al.]{abadi2016tensorflow}
Abadi, M., Barham, P., Chen, J., Chen, Z., Davis, A., Dean, J., Devin, M., Ghemawat, S., Irving, G., Isard, M., et~al.
\newblock $\{$TensorFlow$\}$: a system for $\{$Large-Scale$\}$ machine learning.
\newblock In \emph{12th USENIX symposium on operating systems design and implementation (OSDI 16)}, pp.\  265--283, 2016.

\bibitem[Chen et~al.(2021)Chen, Zeng, Ji, and Yang]{chen2021skyformer}
Chen, Y., Zeng, Q., Ji, H., and Yang, Y.
\newblock Skyformer: Remodel self-attention with gaussian kernel and nystr$\backslash$" om method.
\newblock \emph{Advances in Neural Information Processing Systems}, 34:\penalty0 2122--2135, 2021.

\bibitem[Dettmers(2015)]{dettmers20158}
Dettmers, T.
\newblock 8-bit approximations for parallelism in deep learning.
\newblock \emph{arXiv preprint arXiv:1511.04561}, 2015.

\bibitem[Dettmers et~al.(2021)Dettmers, Lewis, Shleifer, and Zettlemoyer]{dettmers20218}
Dettmers, T., Lewis, M., Shleifer, S., and Zettlemoyer, L.
\newblock 8-bit optimizers via block-wise quantization.
\newblock \emph{arXiv preprint arXiv:2110.02861}, 2021.

\bibitem[Dettmers et~al.(2022)Dettmers, Lewis, Belkada, and Zettlemoyer]{dettmers2022gpt3}
Dettmers, T., Lewis, M., Belkada, Y., and Zettlemoyer, L.
\newblock Gpt3. int8 (): 8-bit matrix multiplication for transformers at scale.
\newblock \emph{Advances in Neural Information Processing Systems}, 35:\penalty0 30318--30332, 2022.

\bibitem[Dettmers et~al.(2023)Dettmers, Svirschevski, Egiazarian, Kuznedelev, Frantar, Ashkboos, Borzunov, Hoefler, and Alistarh]{dettmers2023spqr}
Dettmers, T., Svirschevski, R., Egiazarian, V., Kuznedelev, D., Frantar, E., Ashkboos, S., Borzunov, A., Hoefler, T., and Alistarh, D.
\newblock Spqr: A sparse-quantized representation for near-lossless llm weight compression.
\newblock \emph{arXiv preprint arXiv:2306.03078}, 2023.

\bibitem[Frantar et~al.(2022)Frantar, Ashkboos, Hoefler, and Alistarh]{frantar2022gptq}
Frantar, E., Ashkboos, S., Hoefler, T., and Alistarh, D.
\newblock Gptq: Accurate post-training quantization for generative pre-trained transformers.
\newblock \emph{arXiv preprint arXiv:2210.17323}, 2022.

\bibitem[Gao et~al.(2024)Gao, Tow, Abbasi, Biderman, Black, DiPofi, Foster, Golding, Hsu, Le~Noac'h, Li, McDonell, Muennighoff, Ociepa, Phang, Reynolds, Schoelkopf, Skowron, Sutawika, Tang, Thite, Wang, Wang, and Zou]{eval-harness}
Gao, L., Tow, J., Abbasi, B., Biderman, S., Black, S., DiPofi, A., Foster, C., Golding, L., Hsu, J., Le~Noac'h, A., Li, H., McDonell, K., Muennighoff, N., Ociepa, C., Phang, J., Reynolds, L., Schoelkopf, H., Skowron, A., Sutawika, L., Tang, E., Thite, A., Wang, B., Wang, K., and Zou, A.
\newblock A framework for few-shot language model evaluation, 07 2024.
\newblock URL \url{https://zenodo.org/records/12608602}.

\bibitem[Gholami et~al.(2022)Gholami, Kim, Dong, Yao, Mahoney, and Keutzer]{gholami2022survey}
Gholami, A., Kim, S., Dong, Z., Yao, Z., Mahoney, M.~W., and Keutzer, K.
\newblock A survey of quantization methods for efficient neural network inference.
\newblock In \emph{Low-Power Computer Vision}, pp.\  291--326. Chapman and Hall/CRC, 2022.

\bibitem[Gliwa et~al.(2019)Gliwa, Mochol, Biesek, and Wawer]{gliwa2019samsum}
Gliwa, B., Mochol, I., Biesek, M., and Wawer, A.
\newblock Samsum corpus: A human-annotated dialogue dataset for abstractive summarization.
\newblock \emph{arXiv preprint arXiv:1911.12237}, 2019.

\bibitem[Guan et~al.(2024)Guan, Huang, Su, Huang, Wong, and Yu]{guan2024aptq}
Guan, Z., Huang, H., Su, Y., Huang, H., Wong, N., and Yu, H.
\newblock Aptq: Attention-aware post-training mixed-precision quantization for large language models.
\newblock In \emph{Proceedings of the 61st ACM/IEEE Design Automation Conference}, pp.\  1--6, 2024.

\bibitem[Harper \& Konstan(2015)Harper and Konstan]{harper2015movielens}
Harper, F.~M. and Konstan, J.~A.
\newblock The movielens datasets: History and context.
\newblock \emph{Acm transactions on interactive intelligent systems (tiis)}, 5\penalty0 (4):\penalty0 1--19, 2015.

\bibitem[Kingma \& Ba(2015)Kingma and Ba]{Adam2015}
Kingma, D.~P. and Ba, J.
\newblock Adam: {A} method for stochastic optimization.
\newblock In Bengio, Y. and LeCun, Y. (eds.), \emph{3rd International Conference on Learning Representations, {ICLR} 2015, San Diego, CA, USA, May 7-9, 2015, Conference Track Proceedings}, 2015.
\newblock URL \url{http://arxiv.org/abs/1412.6980}.

\bibitem[Krizhevsky(2009)]{krizhevsky2009learning}
Krizhevsky, A.
\newblock Learning multiple layers of features from tiny images.
\newblock pp.\  32--33, 2009.
\newblock URL \url{https://www.cs.toronto.edu/~kriz/learning-features-2009-TR.pdf}.

\bibitem[Lee et~al.(2024)Lee, Jin, Kim, Kim, and Park]{lee2024owq}
Lee, C., Jin, J., Kim, T., Kim, H., and Park, E.
\newblock Owq: Outlier-aware weight quantization for efficient fine-tuning and inference of large language models.
\newblock In \emph{Proceedings of the AAAI Conference on Artificial Intelligence}, volume~38, pp.\  13355--13364, 2024.

\bibitem[Li et~al.(2023)Li, Chen, and Zhu]{li2023memory}
Li, B., Chen, J., and Zhu, J.
\newblock Memory efficient optimizers with 4-bit states.
\newblock \emph{Advances in Neural Information Processing Systems}, 36:\penalty0 15136--15171, 2023.

\bibitem[Li \& Li(2023)Li and Li]{li2023angle}
Li, X. and Li, J.
\newblock Angle-optimized text embeddings.
\newblock \emph{arXiv preprint arXiv:2309.12871}, 2023.

\bibitem[Lin et~al.(2024)Lin, Tang, Tang, Yang, Chen, Wang, Xiao, Dang, Gan, and Han]{lin2024awq}
Lin, J., Tang, J., Tang, H., Yang, S., Chen, W.-M., Wang, W.-C., Xiao, G., Dang, X., Gan, C., and Han, S.
\newblock Awq: Activation-aware weight quantization for on-device llm compression and acceleration.
\newblock \emph{Proceedings of Machine Learning and Systems}, 6:\penalty0 87--100, 2024.

\bibitem[Linsley et~al.(2019)Linsley, Kim, Veerabadran, and Serre]{linsley2019learninglongrangespatialdependencies}
Linsley, D., Kim, J., Veerabadran, V., and Serre, T.
\newblock Learning long-range spatial dependencies with horizontal gated-recurrent units, 2019.
\newblock URL \url{https://arxiv.org/abs/1805.08315}.

\bibitem[Liu et~al.(2023)Liu, Oguz, Zhao, Chang, Stock, Mehdad, Shi, Krishnamoorthi, and Chandra]{liu2023llm}
Liu, Z., Oguz, B., Zhao, C., Chang, E., Stock, P., Mehdad, Y., Shi, Y., Krishnamoorthi, R., and Chandra, V.
\newblock Llm-qat: Data-free quantization aware training for large language models.
\newblock \emph{arXiv preprint arXiv:2305.17888}, 2023.

\bibitem[Maas et~al.(2011)Maas, Daly, Pham, Huang, Ng, and Potts]{maas-etal-2011-learning}
Maas, A.~L., Daly, R.~E., Pham, P.~T., Huang, D., Ng, A.~Y., and Potts, C.
\newblock Learning word vectors for sentiment analysis.
\newblock In Lin, D., Matsumoto, Y., and Mihalcea, R. (eds.), \emph{Proceedings of the 49th Annual Meeting of the Association for Computational Linguistics: Human Language Technologies}, pp.\  142--150, Portland, Oregon, USA, June 2011. Association for Computational Linguistics.
\newblock URL \url{https://aclanthology.org/P11-1015/}.

\bibitem[Nangia \& Bowman(2018)Nangia and Bowman]{nangia2018listopsdiagnosticdatasetlatent}
Nangia, N. and Bowman, S.~R.
\newblock Listops: A diagnostic dataset for latent tree learning, 2018.
\newblock URL \url{https://arxiv.org/abs/1804.06028}.

\bibitem[Paszke et~al.(2019)Paszke, Gross, Massa, Lerer, Bradbury, Chanan, Killeen, Lin, Gimelshein, Antiga, et~al.]{paszke2019pytorch}
Paszke, A., Gross, S., Massa, F., Lerer, A., Bradbury, J., Chanan, G., Killeen, T., Lin, Z., Gimelshein, N., Antiga, L., et~al.
\newblock Pytorch: An imperative style, high-performance deep learning library.
\newblock \emph{Advances in neural information processing systems}, 32, 2019.

\bibitem[Peng et~al.(2023)Peng, Quesnelle, Fan, and Shippole]{peng2023yarn}
Peng, B., Quesnelle, J., Fan, H., and Shippole, E.
\newblock Yarn: Efficient context window extension of large language models.
\newblock \emph{arXiv preprint arXiv:2309.00071}, 2023.

\bibitem[Radev et~al.(2009)Radev, Muthukrishnan, and Qazvinian]{radev-etal-2009-acl}
Radev, D.~R., Muthukrishnan, P., and Qazvinian, V.
\newblock The {ACL} {A}nthology network corpus.
\newblock In Kan, M.-Y. and Teufel, S. (eds.), \emph{Proceedings of the 2009 Workshop on Text and Citation Analysis for Scholarly Digital Libraries ({NLPIR}4{DL})}, pp.\  54--61, Suntec City, Singapore, August 2009. Association for Computational Linguistics.
\newblock URL \url{https://aclanthology.org/W09-3607/}.

\bibitem[Rae et~al.(2020)Rae, Potapenko, Jayakumar, Hillier, and Lillicrap]{rae-iclr-2020-compressive}
Rae, J.~W., Potapenko, A., Jayakumar, S.~M., Hillier, C., and Lillicrap, T.~P.
\newblock Compressive transformers for long-range sequence modelling.
\newblock In \emph{8th International Conference on Learning Representations, {ICLR} 2020, Addis Ababa, Ethiopia, April 26-30, 2020}. OpenReview.net, 2020.

\bibitem[Shang et~al.(2023)Shang, Yuan, Wu, and Dong]{shang2023pb}
Shang, Y., Yuan, Z., Wu, Q., and Dong, Z.
\newblock Pb-llm: Partially binarized large language models.
\newblock \emph{arXiv preprint arXiv:2310.00034}, 2023.

\bibitem[Su et~al.(2024)Su, Ahmed, Lu, Pan, Bo, and Liu]{su2024roformer}
Su, J., Ahmed, M., Lu, Y., Pan, S., Bo, W., and Liu, Y.
\newblock Roformer: Enhanced transformer with rotary position embedding.
\newblock \emph{Neurocomputing}, 568:\penalty0 127063, 2024.

\bibitem[Sun et~al.(2022)Sun, Dong, Patra, Ma, Huang, Benhaim, Chaudhary, Song, and Wei]{sun2022length}
Sun, Y., Dong, L., Patra, B., Ma, S., Huang, S., Benhaim, A., Chaudhary, V., Song, X., and Wei, F.
\newblock A length-extrapolatable transformer.
\newblock \emph{arXiv preprint arXiv:2212.10554}, 2022.

\bibitem[Tang et~al.(2022)Tang, Han, Guo, Xu, Li, Xu, and Wang]{tang2022image}
Tang, Y., Han, K., Guo, J., Xu, C., Li, Y., Xu, C., and Wang, Y.
\newblock An image patch is a wave: Phase-aware vision mlp.
\newblock In \emph{Proceedings of the IEEE/CVF conference on computer vision and pattern recognition}, pp.\  10935--10944, 2022.

\bibitem[Tay et~al.(2020)Tay, Dehghani, Abnar, Shen, Bahri, Pham, Rao, Yang, Ruder, and Metzler]{tay2020longrangearenabenchmark}
Tay, Y., Dehghani, M., Abnar, S., Shen, Y., Bahri, D., Pham, P., Rao, J., Yang, L., Ruder, S., and Metzler, D.
\newblock Long range arena: A benchmark for efficient transformers, 2020.
\newblock URL \url{https://arxiv.org/abs/2011.04006}.

\bibitem[Tian et~al.(2023)Tian, Bai, Zhao, Wen, and Cao]{tian2023eulernet}
Tian, Z., Bai, T., Zhao, W.~X., Wen, J.-R., and Cao, Z.
\newblock Eulernet: Adaptive feature interaction learning via euler's formula for ctr prediction.
\newblock In \emph{Proceedings of the 46th International ACM SIGIR Conference on Research and Development in Information Retrieval}, pp.\  1376--1385, 2023.

\bibitem[Tian et~al.(2024{\natexlab{a}})Tian, Shi, Wu, Zhao, and Wen]{tian2024rotative}
Tian, Z., Shi, Y., Wu, X., Zhao, W.~X., and Wen, J.-R.
\newblock Rotative factorization machines.
\newblock In \emph{Proceedings of the 30th ACM SIGKDD Conference on Knowledge Discovery and Data Mining}, pp.\  2912--2923, 2024{\natexlab{a}}.

\bibitem[Tian et~al.(2024{\natexlab{b}})Tian, Zhao, Zhang, Zhao, Ma, and Wen]{tian2024eulerformer}
Tian, Z., Zhao, W.~X., Zhang, C., Zhao, X., Ma, Z., and Wen, J.-R.
\newblock Eulerformer: Sequential user behavior modeling with complex vector attention.
\newblock In \emph{Proceedings of the 47th International ACM SIGIR Conference on Research and Development in Information Retrieval}, pp.\  1619--1628, 2024{\natexlab{b}}.

\bibitem[Vanhoucke et~al.(2011)Vanhoucke, Senior, Mao, et~al.]{vanhoucke2011improving}
Vanhoucke, V., Senior, A., Mao, M.~Z., et~al.
\newblock Improving the speed of neural networks on cpus.
\newblock In \emph{Proc. deep learning and unsupervised feature learning NIPS workshop}, volume~1, pp.\ ~4, 2011.

\bibitem[Wang et~al.(2023)Wang, Ma, Dong, Huang, Wang, Ma, Yang, Wang, Wu, and Wei]{wang2023bitnet}
Wang, H., Ma, S., Dong, L., Huang, S., Wang, H., Ma, L., Yang, F., Wang, R., Wu, Y., and Wei, F.
\newblock Bitnet: Scaling 1-bit transformers for large language models.
\newblock \emph{arXiv preprint arXiv:2310.11453}, 2023.

\bibitem[Xiao et~al.(2023)Xiao, Lin, Seznec, Wu, Demouth, and Han]{xiao2023smoothquant}
Xiao, G., Lin, J., Seznec, M., Wu, H., Demouth, J., and Han, S.
\newblock Smoothquant: Accurate and efficient post-training quantization for large language models.
\newblock In \emph{International Conference on Machine Learning}, pp.\  38087--38099. PMLR, 2023.

\bibitem[Yao et~al.(2022)Yao, Yazdani~Aminabadi, Zhang, Wu, Li, and He]{yao2022zeroquant}
Yao, Z., Yazdani~Aminabadi, R., Zhang, M., Wu, X., Li, C., and He, Y.
\newblock Zeroquant: Efficient and affordable post-training quantization for large-scale transformers.
\newblock \emph{Advances in Neural Information Processing Systems}, 35:\penalty0 27168--27183, 2022.

\bibitem[Zhang et~al.(2024)Zhang, Zeng, Wang, and Lu]{zhang2024tinyllama}
Zhang, P., Zeng, G., Wang, T., and Lu, W.
\newblock Tinyllama: An open-source small language model.
\newblock \emph{arXiv preprint arXiv:2401.02385}, 2024.

\bibitem[Zhao et~al.(2023)Zhao, Zhou, Li, Tang, Wang, Hou, Min, Zhang, Zhang, Dong, et~al.]{zhao2023survey}
Zhao, W.~X., Zhou, K., Li, J., Tang, T., Wang, X., Hou, Y., Min, Y., Zhang, B., Zhang, J., Dong, Z., et~al.
\newblock A survey of large language models.
\newblock \emph{arXiv preprint arXiv:2303.18223}, 2023.

\end{thebibliography}
